\documentclass{winnower}

\usepackage[utf8]{inputenc} 
\usepackage[T1]{fontenc}    
\usepackage{hyperref}       
\usepackage{url}            
\usepackage{booktabs}       
\usepackage{amsfonts, dsfont}       
\usepackage{nicefrac}       
\usepackage{microtype}      
\usepackage{natbib}
\usepackage{graphicx}
\usepackage{caption}
\usepackage{subcaption}
\usepackage{amsmath, algorithm, algcompatible}
\usepackage{multirow}
\usepackage{color}
\usepackage{amsfonts,amssymb}
\usepackage{makecell}
\usepackage{amsthm}

\newtheorem{theorem}{Theorem}
\newtheorem{lemma}[theorem]{Lemma}

\input{Qcircuit}

\algnewcommand\INPUT{\item[\textbf{Input:}]}%
\algnewcommand\OUTPUT{\item[\textbf{Output:}]}%
\begin{document}

\title{Learning Hidden Quantum Markov Models}

\author{Siddarth Srinivasan}
\affil{College of Computing, Georgia Institute of Technology, Atlanta, GA 30332, USA}
\author{Geoff Gordon}
\affil{School of Computer Science, Carnegie Mellon University, Pittsburgh, PA 15213, USA}
\author{Byron Boots}
\affil{College of Computing, Georgia Institute of Technology, Atlanta, GA 30332, USA}

\date{\today}

\maketitle

\begin{abstract}
Hidden Quantum Markov Models (HQMMs) can be thought of as quantum probabilistic graphical models that can model sequential data. We extend previous work on HQMMs with three contributions: (1) we show how classical hidden Markov models (HMMs) can be simulated on a quantum circuit, (2) we reformulate HQMMs by relaxing the constraints for modeling HMMs on quantum circuits, and (3) we present a learning algorithm to estimate the parameters of an HQMM from data. While our algorithm requires further optimization to handle larger datasets, we are able to evaluate our algorithm using several synthetic datasets. We show that on HQMM generated data, our algorithm learns HQMMs with the same number of hidden states and predictive accuracy as the true HQMMs, while HMMs learned with the Baum-Welch algorithm require more states to match the predictive accuracy.
\end{abstract}

\section{Introduction}

We extend previous work on Hidden Quantum Markov Models (HQMMs), and propose a novel approach to learning these models from data.
HQMMs can be thought of as a new, expressive class of graphical models that have adopted the mathematical formalism for reasoning about uncertainty from quantum mechanics. We stress that while HQMMs could naturally be implemented on quantum computers, we do not need such a machine for these models to be of value. Instead, 
HQMMs can be viewed as novel models inspired by quantum mechanics that can be run on classical computers. 
In considering these models, we are interested in answering three questions: (1) how can we construct quantum circuits to simulate classical Hidden Markov Models (HMMs); (2) what happens if we take full advantage of this quantum circuit instead of enforcing the classical probabilistic constraints; and (3) how do we learn the parameters for quantum models from data? \\ 



The paper is structured as follows: first we describe related work and provide background on quantum information theory as it relates to our work. Next, we describe the hidden quantum Markov model and compare our approach to previous work in detail, and give a scheme for writing \emph{any} hidden Markov model as an HQMM. Finally, our main contribution is the introduction of a maximum-likelihood-based unsupervised learning algorithm that can estimate the parameters of an HQMM from data. Our implementation is slow to train HQMMs on large datasets, and will require further optimization. Instead, we evaluate our learning algorithm for HQMMs on several simple synthetic datasets by learning a quantum model from data and filtering and predicting with the learned model. We also compare our model and learning algorithm to maximum likelihood for learning hidden Markov models and show that the more expressive HQMM can match HMMs' predictive capability with fewer hidden states on data generated by HQMMs. 

\section{Background} 


\subsection{Related Work}



\indent Hidden Quantum Markov Models  were introduced by \citet{monras2010hidden}, who discussed their relationship to classical HMMs, and parameterized these HQMMs using a set of Kraus operators.  \citet{clark2015hidden} further investigated HQMMs, and showed that they could be viewed as open quantum systems with instantaneous feedback. We 
arrive at the same Kraus operator representation by building a quantum circuit to simulate a classical HMM and then relaxing some constraints. \\

Our work can be viewed as extending previous work by \citet{zhao2010norm} on Norm-observable operator models (NOOM) and \citet{jaeger2000observable} on observable-operator models (OOM). We show that HQMMs can be viewed as complex-valued extensions of NOOMs, formulated in the language of quantum mechanics. We use this connection to adapt the learning algorithm for NOOMs in \citet{noomreport} into the first known learning algorithm for HQMMs, and demonstrate that the theoretical advantages of HQMMs also hold in practice. \\

\citet{schuld2015introduction} and \citet{biamonte2016quantum} provide general overviews of quantum machine learning, and describe relevant work on HQMMs. They suggest that developing algorithms that can learn HQMMs from data is an important open problem. 
We provide just such a learning algorithm 
in Section \ref{sec:la}. \\

Other work at the intersection of machine learning and quantum mechanics includes~\citet{NIPS2016_6401} on quantum perceptron models and learning algorithms. 
\citet{schuld2015simulating} discuss simulating a perceptron on a quantum computer. 



\subsection{Belief States and Quantum States}\label{qstate-section}

Classical discrete latent variable models represent uncertainty 
with a probability distribution 
using a vector $\vec{x}$ whose entries describe the probability of being in the corresponding system state. Each entry is real and non-negative, and the entries sum to 1. 
In general, we refer to the run-time system component that maintains a state estimate of the latent variable as an `observer', and we refer to the observer’s state as a `belief state.' 
A common example is the belief state that results from conditioning on observations in an HMM. \\

In quantum mechanics, the quantum state of a particle $A$ can be written using Dirac notation as $|\psi\rangle_A$, a column-vector in some orthonormal basis (the row-vector is the complex-conjugate transpose $\langle\psi | = (|\psi\rangle)^\dag$) with each entry being the `probability amplitude' corresponding to that system state.  The squared norm of the probability amplitude for a system state is the probability of observing that state, so the sum of squared norms of probability amplitudes over all the system states must be 1 to conserve probability. 
For example, {\small $|\psi\rangle = \begin{bmatrix} \frac{1}{\sqrt{2}} & \frac{-i}{\sqrt{2}} \end{bmatrix}^\dagger$}\normalsize is a valid quantum state, with basis states 0 and 1 having equal probability {\small $\left\|\frac1{\sqrt{2}}\right\|^2 = \left\|\frac{i}{\sqrt{2}}\right\|^2 = \frac12$}\normalsize. However, unlike classical belief states such as {\small $\vec{x} = \begin{bmatrix} \frac12 & \frac12 \end{bmatrix}^T$}\normalsize, where the probability of different states reflects an ignorance of the underlying system, a pure quantum state like the one described above is the \emph{true} description of the system; the system is in both states simultaneously. \\ 

But how can we describe classical mixtures of quantum systems (`mixed states'), where we maintain classical uncertainty about the underlying quantum states? Such information can be captured by a `density matrix.' 
%
Given a mixture of $N$ quantum systems, each with probability $p_i$, the density matrix for this ensemble is defined as follows:
\begin{align}
\hat{\rho} = \sum_i^N p_i|\psi_i\rangle\langle\psi_i|
\end{align}
\normalsize
\noindent The density matrix is the general quantum equivalent of the classical belief state $\vec{x}$ and has diagonal elements representing the probabilities of being in each system state. 
Consequently, the normalization condition is $\text{tr}(\hat{\rho}) = 1$. The off-diagonal elements represent quantum coherences and entanglement, which have no classical interpretation. The density matrix $\hat{\rho}$ can be used to describe the state of any quantum system.\\ 

The density matrix can also be extended to represent the joint state of multiple variables, or that of `multi-particle' systems, to use the physical interpretation. If we have density matrices $\hat{\rho}_A$ and $\hat{\rho}_B$ for two qudits (a $d$-state quantum system, akin to qubits or `quantum bits' which are 2-state quantum systems) $A$ and $B$, we can take the tensor product to arrive at the density matrix for the joint state of the particles, as $\hat{\rho}_{\text{AB}} = \hat{\rho}_A \otimes \hat{\rho}_B$. As a valid density matrix, the diagonal elements of this joint density matrix represent probabilities; $\text{tr}\left(\hat{\rho}_{\text{AB}}\right) = 1$, and the probabilities correspond to the states in the Cartesian product of the basis states of the composite particles. 
In this paper, the joint density matrix will serve as the analogue to classical joint probability distribution, with the off-diagonal terms encoding extra `quantum' information. \\

Given the joint state of a multi-particle system, we can examine the state of just one or few of the particles using the `partial trace' operation, where we trace over the diagonal elements of the particles we wish to disregard. This lets us recover a `reduced density matrix' for a subsystem of interest. The partial trace for a two-particle system $\hat{\rho}_{AB}$ where we trace over the second particle to obtain the state of the first particle is: 
\begin{equation}
\hat{\rho}_A = \text{tr}_B\left(\hat{\rho}_{AB}\right) = \sum_{j} {_B}\langle j|\hat{\rho}_{AB}|j\rangle_B \label{tr}
\end{equation}
For our purposes, this operation will serve as the quantum analogue of classical marginalization. 
Finally, we discuss the quantum analogue of `conditioning' on an observation. In quantum mechanics, the act of measuring a quantum system can change the underlying distribution, i.e., collapses it to the observed state in the measurement basis, and this is represented mathematically by applying von Neumann projection operators (denoted $\hat{P}_{y}$ in this paper) to density matrices describing the system. One can think of the projection operator as a matrix of zeros with ones in the diagonal entries corresponding to observed system states. If we are only observing one part of a larger joint system, the system collapses to the states where that subsystem had the observed result. For example, suppose we have the following density matrix, for a two-state two-particle system with basis $\{|0\rangle_A|0\rangle_B, |0\rangle_A|1\rangle_B, |1\rangle_A|0\rangle_B, |1\rangle_A|1\rangle_B \}$:

\begin{equation}
\hat{\rho}_{AB} = \begin{bmatrix} 0.25 & 0 & 0 & 0\\
							 0 & 0.25 & -0.5 & 0 \\
                             0 & -0.5 & 0.25 & 0 \\
                             0 & 0 & 0 & 0.25 \end{bmatrix}
\end{equation}

Suppose we measure the state of particle $B$, and find it to be in state $|1\rangle_B$. The corresponding projection operator is {\scriptsize $\hat{P}_{1_B} = \begin{bmatrix} 0 & 0 & 0 & 0\\
							 0 & 1 & 0 & 0 \\
                             0 & 0 & 0 & 0 \\
                             0 & 0 & 0 & 1 \end{bmatrix}$} and the collapsed state is now:
{\scriptsize                     
$\hat{\rho}_{AB} = \hat{P}_{1_B}\hat{\rho}_{AB}\hat{P}_{1_B}^\dagger \overset{normalize}{\longrightarrow} 							  \begin{bmatrix} 0 & 0 & 0 & 0\\
							 0 & 0.5 & 0 & 0 \\
                             0 & 0 & 0 & 0 \\
                             0 & 0 & 0 & 0.5 \end{bmatrix}$}. When we trace over particle $A$ to get the state of particle $B$, the result is {\scriptsize $\hat{\rho}_B = \begin{bmatrix} 0 & 0 \\ 0 & 1 \end{bmatrix}$}, reflecting the fact that particle $B$ is now in state $|1\rangle_B$ with certainty. Tracing over particle $B$, we find {\scriptsize $\hat{\rho}_A = \begin{bmatrix} 0.5 & 0 \\ 0 & 0.5 \end{bmatrix}$}, indicating that particle $A$ still has an equal probability of being in either state. Note that measuring $|1\rangle_B$ has changed the underlying distribution of the system $\hat{\rho}_{AB}$; the probability of measuring the state of particle $B$ to be $|0\rangle_B$ is now 0, whereas before measurement we had a $0.25 + 0.25 = 0.5$ chance of measuring $|0\rangle_B$. This is unlike classical probability where measuring a variable doesn't change the joint distribution. We will use this fact when we construct our quantum circuit to simulate HMMs.\\

Thus, if we have an $n$-state quantum system that tracks a particle's evolution, and an $s$-state quantum system that tracks the likelihood of observing various outputs as they depend (probabilistically) on the $n$-state system, upon observing an output $y$, we apply the projection operator $\hat{P}_y$ on the joint system, and trace over the second particle to obtain the $n$-state system conditioned on observation $y$.

\begin{table}[h]\small
\caption{Comparison between classical and quantum representations}
\label{sample-table2}
  \begin{center}
  \begin{tabular}{|ll | ll|}
  \hline
      \multicolumn{2}{|c|}{\bf Classical probability} &
      \multicolumn{2}{c|}{\bf Quantum Analogue}  \\
      \hline\hline
      {Description} & {Representation} & {Representation} & {Description} \\
     \hline
    Belief State & $\vec{x}$  & $\hat{\rho}$ & Density Matrix \\
    Joint Distribution & $\vec{x}_1 \otimes \vec{x}_2$ & $\hat{\rho}_{X_1} \otimes \hat{\rho}_{X_2}$ & Multi-particle Density Matrix \\
    Marginalization & $\vec{x} = \sum_y \vec{x} \otimes \vec{y}$ & $\hat{\rho}_X = \text{tr}_Y(\hat{\rho}_{XY})$ & Partial Trace \\
  {Conditional probability}  & {$P(\vec{x}|y) = \frac{P(y,\vec{x})}{P(y)}$}  & $P(\text{states} \,| y) = \text{tr}_Y(\hat{P}_y\hat{\rho}_{XY}\hat{P}_y^\dagger) $ & Projection + Partial Trace\\ 
  \hline
  \end{tabular}
  \end{center}
\end{table}

\newpage

\subsection{Hidden Markov Models} \label{sec:hmm}

Classical Hidden Markov Models (HMMs) are graphical models used to model dynamic processes that exhibit Markovian state evolution.  
Figure \ref{fig:hmm} depicts a classical HMM, where the transition matrix $\mathbf{A}$ and emission matrix $\mathbf{C}$ are column-stochastic matrices that determine the Markovian hidden state-evolution and observation probabilities respectively. 
Bayesian inference can be used to track the evolution of the hidden variable.

\begin{figure}[h!]
\centerline{ \includegraphics[scale=0.15]{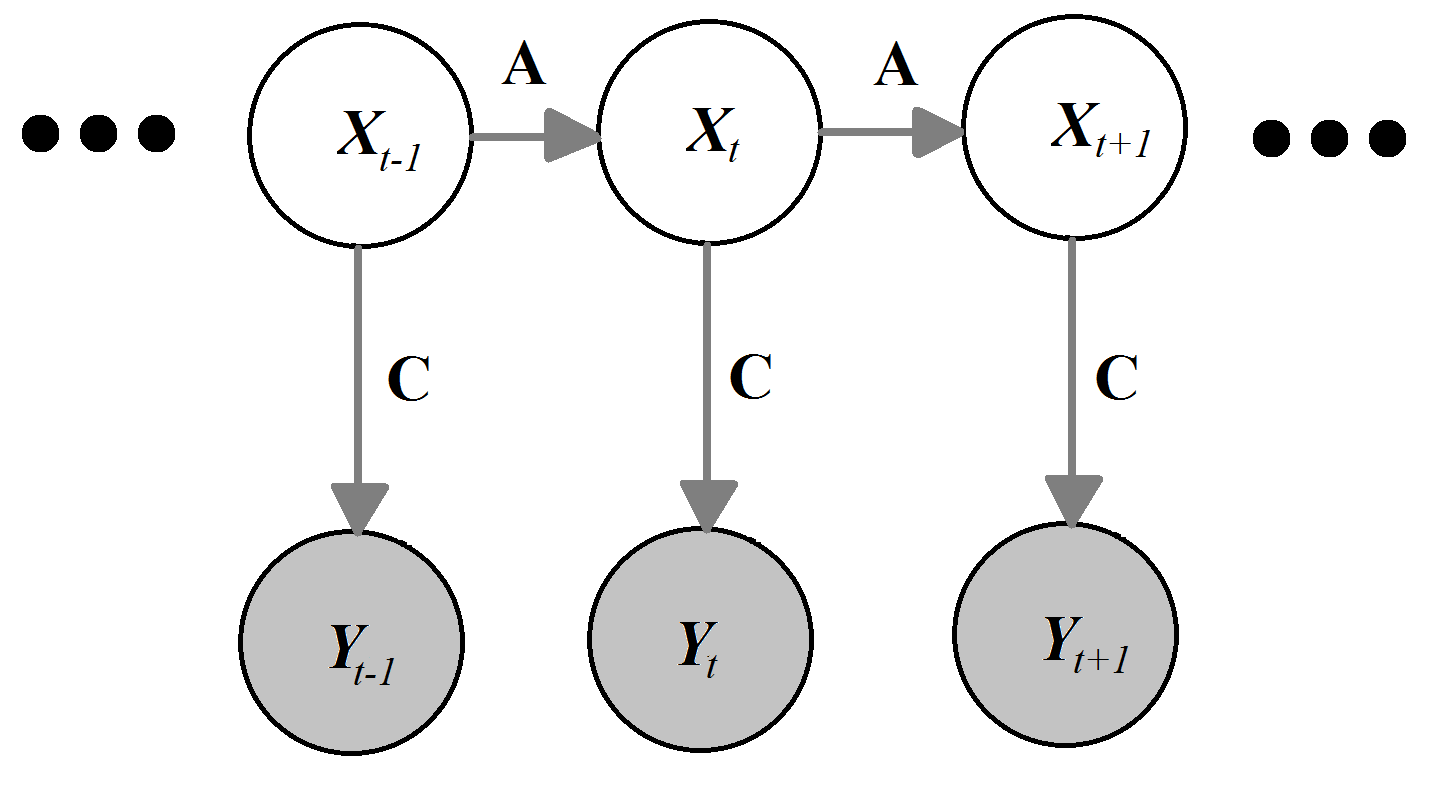}}
\caption{Hidden Markov Model}
    \label{fig:hmm}
\end{figure}


The belief state at time $t$ is a probability distribution over states, and prior to any observation is written as:
\begin{equation}
\vec{x}_t' = \mathbf{A}\vec{x}_{t-1}
\end{equation}
The probabilities of observing each output at time $t$ is given by the vector $\vec{s}$:
\begin{equation}
\vec{s}_t = \mathbf{C}\vec{x}_t'= \mathbf{CA}\vec{x}_{t-1}
\end{equation}
We can use Bayesian inference to write the belief state vector after conditioning on observation $y$:
\begin{equation} \vec{x}_t = \frac{\text{diag}(\mathbf{C}_{(y,:)})\mathbf{A}\vec{x}_{t-1}}{\mathds{1}^T\text{diag}(\mathbf{C}_{(y,:)})\mathbf{A}\vec{x}_{t-1}} 
\label{eq:hmmac}
\end{equation}
where $\text{diag}(\mathbf{C}_{(y,:)})$ is a diagonal matrix with the entries of the $y$th row of $\mathbf{C}$ along the diagonal, and the denominator renormalizes the vector $\vec{x}_t$. \\

An alternate representation of the Hidden Markov Model uses `observable' operators (\citet{jaeger2000observable}). Instead of using the matrices $\mathbf{A}$ and $\mathbf{C}$, we can write $\mathbf{T}_y = \text{diag}(\mathbf{C}_{(y,:)})\mathbf{A}$. There is a different operator $\mathbf{T}_y$ for each possible observable output $y$ and $[\mathbf{T}_y]_{ij} = P(y; i_t|j_{t-1})$. We can then rewrite Equation \ref{eq:hmmac} as:
\begin{equation} \vec{x}_t = \frac{\mathbf{T}_y\vec{x}_{t-1}}{\mathds{1}^T\mathbf{T}_y\vec{x}_{t-1}} \tag{4}
\end{equation}
If we observe outputs $y_1, \ldots, y_n$, we apply $\mathbf{T}_n\ldots\mathbf{T}_1\vec{x}$ and take the sum of the resulting vector to find the probability of observing the sequence, or renormalize to find the belief state after the final observation.

\section{Hidden Quantum Markov Models}
\subsection{A Quantum Circuit to Simulate HMMs}\label{qcirc-section}

Let us now contrast state evolution in quantum systems with state evolution in HMMs. 
%
The quantum analogue of observable operators is a set of non-trace-increasing Kraus operators \{$\hat{K}_i$\} that are completely positive (CP) linear maps. Trace-preserving Kraus operators  $\sum_i^N\hat{K}^\dagger_i\hat{K}_i = \mathbb{I}$, can map a density operator to another density operator. Trace-decreasing Kraus operators $\sum_i^N\hat{K}^\dagger_i\hat{K}_i < \mathbb{I}$, represent operations on a smaller part of a quantum system that can allow probability to `leak' to other states that aren't being considered. This paper will formulate problems such that all sets of Kraus operators are trace-preserving. When there is only one operator in the set, i.e., $\hat{U}$ such that $\hat{U}^\dagger\hat{U} = \mathbb{I}$, then $\hat{U}$ is a unitary matrix. 
Unitary operators generally model the evolution of the `whole' system, which may be high-dimensional. But if we care only about tracking the evolution of a smaller sub-system, which may interact with its environment, we can use Kraus operators. The most general quantum operation that can be performed on a density matrix is $\hat{\rho}' = \frac{\sum_i^M K^\dagger_i\hat{\rho}K_i}{\text{tr}\left(\sum_i^M K^\dagger_i\hat{\rho}K_i\right)}$, where the denominator re-normalizes the density matrix.\\ 

Now, how do we simulate classical HMMs on quantum circuits with qudits, where computation is done using unitary operations? 
There is no general way to convert column-stochastic transition and emission matrices to unitary matrices, so we 
prepare `ancilla'  particles and construct 
unitary matrices (see Algorithm \ref{alg:coltoun}) to act on the joint state. 
We then trace over one particle to obtain the state of the other.

\begin{algorithm*}[h]
\caption{$s \times n$ Column-Stochastic Matrix to $ns \times ns$ Unitary Matrix}
\begin{algorithmic}[1]
\INPUT $s \times n$ Column-Stochastic Matrix ${\bf A}$
\OUTPUT $ns \times ns$ block diagonal Unitary Matrix $\hat{U}$ with $n$ blocks of $s \times s$ unitary matrices, zeros everywhere else
\STATE \textbf{Construct an $s \times s$ unitary matrix from each column of $A$}: Let $c_i$ denote the $i$th column of $\mathbf{A}$. First create an $s \times s$ matrix whose each row is the square root of column $c_i$. Find the null space of this matrix, and you will get the $s-1$ vectors that are linearly independent of $c_i$. Make $c_i$ the first column, and the remaining $s-1$ vectors the other columns of an $s \times s$ matrix.
\STATE \textbf{Stack each $s \times s$ matrix on a diagonal}: Follow step 1 for each column of $A$, and obtain $n$ unitary matrices of dimension $s \times s$. Create a block diagonal matrix with each of these smaller unitary matrices along the diagonal, and you will obtain an $ns \times ns$ dimensional unitary matrix $\hat{U}$.

\STATE \bf{Note:} The unitary operator constructed here is designed to be applied on a density matrix tensored with an environment density matrix prepared with zeros everywhere except $\hat{\rho}_{1,1} = 1$.
\end{algorithmic}
\label{alg:coltoun}
\end{algorithm*}


Figure \ref{fig:qcirchmm} illustrates a quantum circuit constructed with these unitary matrices. By preparing the `ancilla' states $\hat{\rho}_{X_{t}}$ and $\hat{\rho}_{Y_{t}}$ appropriately  (i.e., entirely in system state 1, represented by a density matrix of zeros except $\hat{\rho}_{1,1}= 1$), we 
construct $\hat{U}_1$ and $\hat{U}_2$ from transition matrix $\mathbf{A}$ and emission matrix $\mathbf{C}$, respectively.  
$\hat{U}_1$ evolves $\left( \hat{\rho}_{t-1} \otimes \hat{\rho}_{X_t}\right)$ to perform Markovian transition, while $\hat{U}_2$ updates $\hat{\rho}_{Y_t}$ to contain the probabilities of measuring each observable output. At runtime, we measure $\hat{\rho}_{Y_t}$ which changes the joint distribution of $\hat{\rho}_{X_t} \otimes \hat{\rho}_{Y_t}$ to give the updated conditioned state $\hat{\rho}_t$. Mathematically, this is equivalent to applying a projection operator on the joint state and tracing over $\hat{\rho}_{Y_t}$.  
Thus, the forward algorithm corresponding to Figure \ref{fig:qcirchmm} that explicitly models a hidden Markov Model on a quantum circuit can be written as:
\small
\begin{equation} \hspace{-.3mm}\hat{\rho}_t \propto \text{tr}_{\hat{\rho}_{Y_{t}}}\left(\hat{P}_y\hat{U}_{2}\left(\text{tr}_{\hat{\rho}_{t-1}}(\hat{U}_1(\hat{\rho}_{t-1} \otimes \hat{\rho}_{X_{t}})\hat{U}_1^\dagger) \otimes \hat{\rho}_{Y_{t}}\right)\hat{U}_{2}^\dagger\hat{P}_y^\dagger\right)\hspace{-.7mm}
\end{equation}
\normalsize
We can simplify this circuit to use Kraus operators acting on the lower-dimensional state space of $\hat{\rho}_{X_t}$.  
Since we always prepare $\hat{\rho}_{Y_{t}}$ in the same state, the operation $\hat{U}_2$ on the joint state of $\hat{\rho}_{X_{t}} \otimes \hat{\rho}_{Y_{t}}$ followed by the application of the projection operator $\hat{P}_y$ can be more concisely written as a Kraus operator on just $\hat{\rho}_{X_{t}}$, so that we need only be concerned with representing how the particle $\hat{\rho}_{X_t}$ evolves. We would need to construct a set of Kraus operators $\{\hat{K}_{y}\}$ for each observable output $y$, such that $\sum_y (\hat{K}_{y})^\dagger (\hat{K}_{y}) = \mathds{I}$.\\

Tensoring with an ancilla qudit and tracing over a qudit 
can be achieved with an $ns \times n$ matrix $W$ and an $n \times ns$ matrix $V_y$ respectively, since we always prepare our ancilla qudits in the same state (details on constructing these matrices can be found in the Appendix), so that:
\begin{equation}
\begin{split}
\hat{\rho}_{X_t} \otimes \hat{\rho}_{Y_t} &\longrightarrow W\hat{\rho}_XW^\dagger \\
{\text{tr}_{\hat{\rho}_{Y_t}}\left(\hat{P}_{y}\hat{U}_2 W\hat{\rho}_{X_t} 
W^\dagger\hat{U}_2^\dagger\hat{P}_{y}^\dagger\right)} &\longrightarrow { V_y\hat{P}_{y}\hat{U}_2 W\hat{\rho}_{X_t} W^\dagger\hat{U}_2^\dagger\hat{P}_{y}^\dagger V^\dagger_y}
\end{split}
\end{equation}

\begin{figure}[h]
\centering
  \begin{subfigure}[b]{0.4\textwidth}
    \mbox{\Qcircuit @C=1em @R=1em {
\lstick{\hat{\rho}_{t-1}} & \multigate{1}{\hat{U}_1} & \qw & \qw &  &  & & \\
\lstick{\hat{\rho}_{X_{t}}} & \ghost{\mathcal{F}} & \qw & \multigate{1}{\hat{U}_2} & \qw & \qw &  {\hat{\rho}_t} \\
\lstick{\hat{\rho}_{Y_{t}}} & \qw & \qw & \ghost{\mathcal{F}} & \measureD{}\\
}}
		\caption{Full Quantum Circuit to implement HMM}		
        \label{fig:qcirchmm}
  \end{subfigure}
  \hspace{0.5in}
  \begin{subfigure}[b]{0.4\textwidth}
    \mbox{\Qcircuit @C=1em @R=1em {
\lstick{\hat{\rho}_{t-1}} & \multigate{1}{\hat{U}_1} & \qw & \qw & & & & \\
\lstick{\hat{\rho}_{X_{t}}} & \ghost{\mathcal{F}} & \qw & \gate{\hat{K}_{y_{t-1}}} & \qw & \qw &  {\hat{\rho}_t} \\
}}		
        \caption{Simplified scheme to implement HMM}
        \label{fig:qcirchmm2}
  \end{subfigure}
\label{fig:qcircall}
\caption{HMM implementation on quantum circuits}
\end{figure}
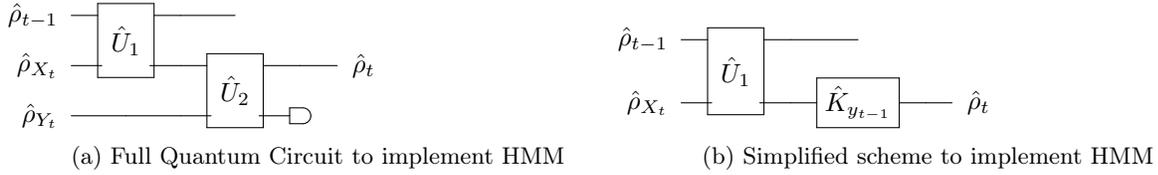

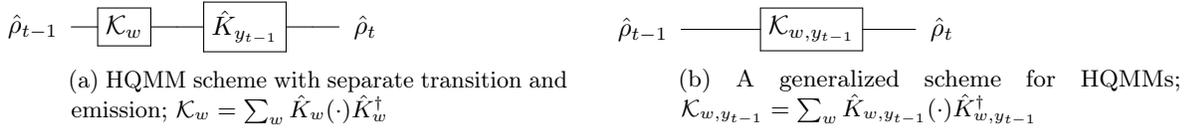
\begin{figure}[h!]
\centering
  \begin{subfigure}[b]{0.4\textwidth}
    \mbox{\Qcircuit @C=1em @R=1em {
\lstick{\hat{\rho}_{t-1}} & \gate{\mathcal{K}_w} & \qw & \gate{\hat{K}_{y_{t-1}}} & \qw & \qw &  {\hat{\rho}_t} \\}}
\caption{HQMM scheme with separate transition and emission; $\mathcal{K}_w = \sum_w\hat{K}_w(\cdot)\hat{K}_w^\dagger$}		
        \label{fig:hqmm1}
  \end{subfigure}
  \hspace{0.5in}
  \begin{subfigure}[b]{0.4\textwidth}
    \mbox{\Qcircuit @C=1em @R=1em {
\lstick{\hat{\rho}_{t-1}} & \qw & \qw & \gate{\mathcal{K}_{w, y_{t-1}}} & \qw & \qw &  {\hat{\rho}_t} \\
}}		
        \label{fig:hqmm2}
        \caption{A generalized scheme for HQMMs; $\mathcal{K}_{w,y_{t-1}} = \sum_w\hat{K}_{w,y_{t-1}}(\cdot)\hat{K}_{w,y_{t-1}}^\dagger$}
  \end{subfigure}
  \label{fig:hqmm}
\caption{Quantum schemes implementing classical HMMs}
\end{figure}

We can then construct Kraus operators such that $\hat{K}_{y} = V_y\hat{P}_{y}\hat{U}_2 W$. 
Figure \ref{fig:qcirchmm2} shows this updated circuit, where $\hat{U}_1$ is still the quantum implementation of the transition matrix and $\hat{K}_{y_t}$ is the quantum implementation of the Bayesian update after observation. This scheme to model a classical HMM can be written as:
\begin{equation} \hat{\rho}_t = \frac{\hat{K}_{y_{t-1}}\left(\text{tr}_{\hat{\rho}_{t-1}}(\hat{U}_1(\hat{\rho}_{t-1} \otimes \hat{\rho}_{X_{t}})\hat{U}_1^\dagger)\right)\hat{K}_{y_{t-1}}^\dagger}{\text{tr}\left(\hat{K}_{y_{t-1}}\left(\text{tr}_{\hat{\rho}_{t-1}}(\hat{U}_1(\hat{\rho}_{t-1} \otimes \hat{\rho}_{X_{t}})\hat{U}_1^\dagger)\right)\hat{K}_{y_{t-1}}^\dagger\right)}
\end{equation}

We can similarly 
simplify $\hat{U}_1$ to a set of Kraus operators. 
We write the unitary operation $\hat{U}_1$ in terms of a set of $n$ Kraus operators $\{\hat{K}_w\}$ as if we were to measure $\hat{\rho}_{t-1}$ immediately after the operation $\hat{U}_1$. However, instead of applying one Kraus operator associated with measurement as we do with Figure \ref{fig:qcirchmm2}, we sum over all of $n$ possible `observations', as if to `ignore' the observation on $\hat{\rho}_{t-1}$. Post-multiplying each Kraus operator in $\{\hat{K}_w\}$ with each operator in $\{\hat{K}_{y}\}$, we have a set of Kraus operators $\{\hat{K}_{w_y,y}\}$ that can be used to model a classical HMM as follows (the full procedure is described in Algorithm~\ref{alg:hqmm}):
\begin{equation}
\hat{\rho}_t = \frac{\sum_{w_y} \hat{K}_{w_y,y_{t-1}}\hat{\rho}_{t-1}\hat{K}_{w_y,y_{t-1}}^\dagger}{\text{tr}\left(\sum_{w_y} \hat{K}_{w_y,y_{t-1}}\hat{\rho}_{t-1}\hat{K}_{w_y,y_{t-1}}^\dagger\right)}
\label{eq:hqmm}
\end{equation}

We believe this procedure to be a useful illustration of performing classical operations on graphical models using quantum circuits. In practice, we needn't construct the Kraus operators in this peculiar fashion to simulate HMMs; an equivalent but simpler approach is to construct observable operators $\{ {\bf T}_y \}$ from transition and emission matrices as described in section \ref{sec:hmm}, and set the $w$th column of $\hat{K}_{w_y, y}^{(:, w)} = \sqrt{{\bf T}_y^{(:,w)}}$, with all other entries being zero. This ensures $\sum_{w_y, y} \hat{K}^\dagger_{w_y,y} K_{w_y,y} = \mathds{I}$.

\begin{algorithm}[h]
    \caption{Simulating Hidden Markov Models with HQMMs}
    \label{alg:hqmm}
  \begin{algorithmic}[1]
    \INPUT Transition Matrix $\mathbf{A}$ and Emission Matrix $\mathbf{C}$
    \OUTPUT Belief State  as $diag(\hat{\rho})$, or $P(y_1,\ldots,y_n|\mathds{D})$ where $\mathds{D}$ is the HMM
        \STATE \textbf{Initialization:}
    	\STATE \,\,\,\, Let $s=\#\text{outputs}$, $n=\#\text{hidden states}$, $y_t =$ observed output at time $t$
    	\STATE \,\,\,\, Prepare density matrix $\hat{\rho}$ in some initial state. $\hat{\rho} = diag(\pi)$ if priors $\pi$ are known.
        \STATE \,\,\,\, Construct unitary matrices $\hat{U}_1$ and $\hat{U}_2$ from $\mathbf{A}$ and $\mathbf{C}$ respectively using Algorithm \ref{alg:coltoun} (in appendix)
        \STATE \,\,\,\, Using $\hat{U}_1$ and $\hat{U}_2$, construct a set of $n$ Kraus Operators $\{\hat{K}_w\}$ and $s$ Kraus operators $\{\hat{K}_{y}\}$, with $\hat{K}_{w} = V_w\hat{U}_1 W$ and $\hat{K}_{y} = V_y\hat{P}_{y}\hat{U}_2 W$ and combine them into a set $\{ \hat{K}_{w_y,y} \}$ with $\hat{K}_{w_y,y} = \hat{K}_y\hat{K}_w$. (Matrix $W$ tensors with an ancilla, Matrix $V_y$ carries out a trivial partial trace operation and summing over $V_w$ for all $w$ carries out the proper partial trace operation. Details in appendix).

    \FOR{$t=1:T$}
      \STATE $\hat{\rho}_{t+1} \leftarrow \sum_{w_y} \hat{K}_{w,y_{t}}\hat{\rho}_{t-1}(\hat{K}_{w_y,y_i})^\dagger$
	\ENDFOR
    \STATE tr($\hat{\rho}_T$) gives the probability of the sequence; renormalizing $\hat{\rho}_T$ gives the belief state on the diagonal.
  \end{algorithmic}
\end{algorithm}

\subsection{Formulating HQMMs}\label{form_sect}

\citet{monras2010hidden} formulate Hidden Quantum Markov Models by defining a set of Kraus operators $\{\hat{K}_{w_y,y}\}$, where each observable $y$ has $w_y$ associated Kraus operators acting on a state with hidden dimension $n$, and they form a complete set such that $\sum_{w,y} \hat{K}^\dagger_{w,y} \hat{K}_{w,y} = \mathds{I}$. The update rule for a quantum operation is exactly the same as Equation \ref{eq:hqmm}, which we arrived at by first constructing a quantum circuit to simulate HMMs with known parameters and then constructing operators $\{\hat{K}_{w,y}\}$ in a very peculiar way. 
The process outlined in the previous section is a particular parameterization of HQMMs to model HMMs. If 
we let the operators $\hat{U}_1$ and $\hat{U}_2$ be any unitary matrices, or the Kraus operators be any set of complex-valued matrices that satisfy $\sum_{w_y, y} \hat{K}^\dagger_{w_y,y} K_{w_y,y} = \mathds{I}$, then we have a general and fully quantum HQMM.\\

Indeed, Equation \ref{eq:hqmm} gives the forward algorithm for HQMMs. To find the probability of emitting an output $y$ given the previous state $\hat{\rho}_{t-1}$, we simply take the trace of the numerator in Equation \ref{eq:hqmm}, i.e., $p(y_{t}|\hat{\rho}_{t-1}) = \text{tr}\left(\sum_{w_y} \hat{K}_{w_y,y_{t-1}}\hat{\rho}_{t-1}\hat{K}_{w_y,y_{t-1}}^\dagger\right)$.\\

The number of parameters for a HQMM is determined by the number of latent states $n$, outputs $s$, and Kraus operators associated with an output $w$. 
To exactly simulate HMM dynamics with an HQMM, we need $w=n$ as per the derivation above. However, this constraint need not hold for a general HQMM, which can have any number of Kraus operators we apply and sum for a given output. $w$ can also be thought of as the dimension of the ancilla $\hat{\rho}_{X_t}$ that we tensor with in Figure \ref{fig:qcirchmm} before the unitary operation $\hat{U}_1$. Consequently, if we set $w=1$, we do not tensor with an additional particle, but model the evolution of the original particle as unitary. 
In all, a HQMM requires learning $n^2sw$ parameters, which is a factor $w$ times more than a HMM with the observable operator representation which has $n^2s$ parameters. The canonical representation of HMMs with with an $n\times n$ transition matrix and an $s \times n$ emission matrix has $n^2 + ns$ parameters.\\

 HQMMs can also be seen as a complex-valued extension of norm-observable operator models defined by \citet{zhao2010norm}. Indeed, the HQMM we get by applying Algorithm $\ref{alg:hqmm}$ on a HMM is also a valid NOOM (allowing for multiple operators per output), 
implying that HMMs can be simulated by NOOMs. We can also state that both HMMs and NOOMs can be simulated by HQMMs (the latter is trivially true). While \cite{zhao2010norm} show that any NOOM can be written as an OOM, the exact relationship between HQMMs and OOMs is not straightforward owing to the complex entries in HQMMs and requires further investigation. 

\section{An Iterative Algorithm For Learning HQMMs} \label{sec:la}

We present an iterative maximum-likelihood algorithm to \emph{learn} Kraus operators to model sequential data using an HQMM. 
Our algorithm is general enough that it can be applied to \emph{any} quantum version of a classical machine learning algorithm for which the loss is defined in terms of the Kraus operators to be learned. \\


We begin by writing the likelihood of observing some sequence $y_1,\ldots,y_T$. Recall that for a given output $y$, we apply the $w$ Kraus operators associated with that observable in the `forward' algorithm, as $\sum_{w_{y}}\hat{K}_{w_{y}, y}(\cdot)\hat{K}_{w_{y}, y}$. If we do not renormalize the density matrix after applying these operators, the diagonal entries contain the joint probability of the corresponding system states and observing the associated sequence of outputs. The trace of this un-normalized density matrix gives the probability of observing $y$ since we have summed over (i.e., marginalized) all the `hidden' states. Thus, the general log-likelihood of a sequence of length $n$ being predicted by a HQMM where each observable $y$ has $w_y$ associated Kraus operators is: 

\begin{equation} \mathcal{L} = \ln \text{tr}\left(\sum_{w_{y_n}}\hat{K}_{w_{y_n},y_n}\ldots \left(\sum_{w_{y_1}}\hat{K}_{w_{y_1}, y_1}\hat{\rho}_0\hat{K}_{w_{y_1},y_1}^\dagger\right)\ldots \hat{K}_{w_{y_n},y_n}^\dagger\right) 
\end{equation}

It is not straightforward to directly maximize this log-likelihood using gradient descent; we must preserve the Kraus operator constraints and long sequences can quickly lead to underflow issues.  Our approach is to learn a $nsw \times n$ matrix $\kappa^*$, which is essentially the set of $ws$ Kraus operators $\{\hat{K}_{w, y}\}$ of dimension $n \times n$, stacked vertically. 
The Kraus operators constraint requires $\sum_s \hat{K}^\dagger_s \hat{K}_s = \mathds{I}$, which implies $\kappa^\dagger\kappa = \mathds{I}$, where the columns of $\kappa$ are orthonormal. \\

\begin{algorithm}[h]
\caption{Iterative Learning Algorithm for Hidden Quantum Markov Models}\label{alg:learn2}
\begin{algorithmic}[1]
\INPUT A $M \times \ell$ matrix $Y$, where $M$ is the number of data points and $\ell$ is the length of a stochastic sequence to be modeled.
\OUTPUT A set of $ws$ of $n \times n$ Kraus operators $\{\hat{K}_{w, s}\}$ that maximize the log-likelihood of the data, where $n$ is the dimension of the hidden state, $s$ is the number of outputs, and $w$ is the number of operators per outputs.
\STATE \textbf{Initialization}: Randomly generate a set of $ws$ Kraus operators $\{\hat{K}_{w,s}\}$ of dimension $n \times n$, and stack them vertically to obtain a matrix $\kappa$ of dimension $nsw \times n$. Let $b$ be the batch size, $B$ the total number of batches to process, and $Y_b$ a $b \times \ell$ matrix of randomly chosen data samples. Let $num\_iterations$ be the number of iterations spent modifying $\kappa$ to maximize the likelihood of observing $Y_b$.

\FOR{\text{batch} = $1$:$B$}
	\STATE Randomly select $b$ sequences to process, and construct matrix $Y_b$
	\FOR{$it = 1:num\_iterations$}
    	\STATE Randomly select rows $i$ and $j$ of $\kappa$ to modify, $i < j$
        \STATE Find $\vec{w} = (\phi, \psi, \delta,\theta)$ that maximises the log-likelihood of $Y_b$ under the following update, and update:
        \begin{align*}
        \kappa^i &\leftarrow \left(e^{\nicefrac{i\phi}2}e^{i\psi}\cos(\theta)\right)\kappa^i + \left(e^{\nicefrac{i\phi}2}e^{i\delta}\sin(\theta)\right)\kappa^j \\
\kappa^j &\leftarrow \left(-e^{\nicefrac{i\phi}2}e^{-i\delta}\sin(\theta)\right)\kappa^i  + \left(e^{\nicefrac{i\phi}2}e^{-i\psi}\cos(\theta)\right)\kappa^j
\end{align*}
    \ENDFOR
\ENDFOR
\end{algorithmic}
\end{algorithm}

Let $\kappa$ be our guess and $\kappa^*$ be the \emph{true} matrix of stacked Kraus operators that maximizes the likelihood under the observed data. Then, there must exist some unitary operator $\hat{U}$ that maps $\kappa$ to $\kappa^*$, i.e., $\kappa^* = \hat{U}\kappa$. Our goal is now to find the matrix $\hat{U}$. 
To do this, we use the fact that the matrix $\hat{U}$ can written as the product of simpler matrices ${\bf H}(i, j, \theta,\phi,\psi,\delta)$ (see appendix for proof), where

\begin{equation}
{\bf H}(i,j,\theta,\phi,\psi,\delta) = \begin{bmatrix} 1 & \cdots & 0 & \cdots & 0 & \cdots & 0 \\ \vdots & \ddots & \vdots &  &  \vdots & & \vdots \\ 0 & \cdots & e^{\nicefrac{i\phi}2}e^{i\psi}\cos\theta & \cdots & e^{\nicefrac{i\phi}2}e^{i\delta}\sin\theta & \cdots & 0 \\ \vdots &  & \vdots & \ddots  &  \vdots & & \vdots \\ 0 & \cdots & -e^{\nicefrac{i\phi}2}e^{-i\delta}\sin\theta & \cdots & e^{\nicefrac{i\phi}2}e^{-i\psi}\cos\theta & \cdots & 0 \\ \vdots & & \vdots &  &  \vdots & \ddots & \vdots \\ 0 & \cdots & 0 & \cdots & 0 & \cdots & 1\end{bmatrix}
\end{equation}

$i$ and $j$ specify the two rows in the matrix with the non-trivial entries, and the other paramters $\theta,\phi,\psi,\delta$ are angles that parameterize the non-trivial entries. The ${\bf H}$ matrices can be thought of as Givens rotations generalized for complex-valued unitary matrices. Applying such a matrix ${\bf H}(i,j,\theta,\phi,\psi,\delta)$ on $\kappa$ has the effect of combining rows $i$ and $j$ ($i < j$) of $\kappa$ like so:
\begin{equation} \label{eq1}
\begin{split}
\kappa^i &\leftarrow \left(e^{\nicefrac{i\phi}2}e^{i\psi}\cos(\theta)\right)\kappa^i + \left(e^{\nicefrac{i\phi}2}e^{i\delta}\sin(\theta)\right)\kappa^j\\
\kappa^j &\leftarrow \left(-e^{\nicefrac{i\phi}2}e^{-i\delta}\sin(\theta)\right)\kappa^i + \left(e^{\nicefrac{i\phi}2}e^{-i\psi}\cos(\theta)\right)\kappa^j
\end{split}
\end{equation}
Now the problem becomes one of identifying the sequence of ${\bf H}$ matrices that can take $\kappa$ to $\kappa^*$. Since the optimization is non-convex and the ${\bf H}$ matrices need not commute, we are not guaranteed to find the global maximum. Instead, we look for a local-max $\kappa^*$ that is reachable by only multiplying ${\bf H}$ matrices that increase the log-likelihood. To find this sequence, we iteratively find the parameters $(i,j,\theta,\phi,\psi,\delta)$ that, if used in equation \ref{eq1}, would increase the log-likelihood. To perform this optimization, we use the \texttt{fmincon} function in MATLAB that uses interior-point optimization. 
It can also be computationally expensive to find the the best rows $i,j$ to swap at a given step, so in our implementation, we randomly pick the rows $(i,j)$ to swap. See Algorithm \ref{alg:learn2} for a summary. We believe more efficient implementations are possible, but we leave this to future work.

%


\section{Experimental Results}

In this section, we evaluate the performance of our learning algorithm on simple synthetic datasets, and compare it to the performance of Expectation Maximization for HMMs (\cite{rabiner1989tutorial}). We judge the quality of the learnt model using its Description Accuracy (DA) (\cite{noomreport}), defined as:
\begin{equation}
DA = f\left(1 + \frac{\log_s P(Y|\mathds{D})}{\ell}\right)
\end{equation}
where $\ell$ is the length of the sequence, $s$ is the number of output symbols in the sequence, $Y$ is the data, and $\mathds{D}$ is the model. Finally, the function $f(\cdot)$ is a non-linear function that takes the argument from $(-\infty, 1]$ to $(-1,1]$:
\begin{equation} f(x) = \left\{
\begin{array}{ll}
      x & x\geq 0 \\
      \frac{1-e^{-0.25x}}{1+e^{-0.25x}} & x < 0 \\
\end{array} 
\right. 
\end{equation}
If $DA=1$, the model perfectly predicted the stochastic sequence, while $DA>0$ would mean that the model predicted the sequence better than random.  \\

In each experiment, we generate 20 training sequences of length 3000, and 10 validation sequences of length 3000, with a `burn-in' of 1000 to disregard the influence of the starting distribution. We use QETLAB (a MATLAB Toolbox developed by \cite{qetlab}) to generate random HQMMs. We apply our learning algorithm once to learn HQMMs from data and report the DA. We use the Baum-Welch algorithm implemented in the \texttt{hmmtrain} function from MATLAB's Statistics and Machine Learning Toolbox to learn HMM parameters. When training HMMs, we train 10 models and report the best DA.\\

We found that starting with a batch size of 1 with 5-6 iterations to get close to the local maximum, and then increasing the batch size to 3-4 and smaller $num\_iterations \sim 3$ was a good way to reach convergence. We also find that training models with $w > 1$ becomes very slow; when $w=1$, to compute the log-likelihood, we can simply take the product of all the Kraus operators corresponding to the observed sequence, and apply it on either side of the density matrix. However, with $w \geq 2$, we have to perform a sum over the $w$ Kraus operators corresponding to a given observation, before we can apply the next set of Kraus operators. \\

The first experiment compares learned models on data generated by a valid `probability clock' NOOM/HQMM model (\cite{noomreport}) that theoretically cannot be modeled by a finite-dimensional HMM. The second experiment considers data generated by the 2-state, 4-output HQMM proposed in \cite{monras2010hidden}, which requires at least 3 hidden states to be modeled with an HMM. The third experiment is performed on data generated by physically motivated, fully quantum 2-state, 6-output HQMM requiring at least 4 classical states for HMMs to model, and can be seen as an extension of the \cite{monras2010hidden} model. Finally, we compare the performance of our algorithm with EM for HMMs on data that was generated by a hand-written HMM. These experiments are meant to showcase the greater expressiveness of HQMMs compared with HMMs. While we see mixed performance on HMM-generated data, 
%
%
we are able to empirically demonstrate that on the HQMM-generated datasets, our algorithm is able to learn an HQMM that can better predict the generated data than EM for classical HMMs with fewer hidden states. 


\subsection{Probability Clock}

\cite{zhao2010norm} describes a 2-hidden state, 2-observable NOOM `probability clock,' where the probability of generating an observable $a$ changes periodically with the length of the sequence of $a$s preceding it, and cannot be modeled with a finite-dimensional HMM: 

{\small
\begin{equation}
\hat{K}_{1,1} = \begin{pmatrix} 0.6\cos(0.6) & -\sin(0.6) \\ 0.6\sin(0.6) & \cos(0.6)\end{pmatrix}  \,
\hat{K}_{1,2} = \begin{pmatrix} 0.8 & 0 \\ 0 & 0\end{pmatrix}\hspace{-2mm}
\end{equation}}
\normalsize

This is a valid HQMM since  $\sum_{y=1}^{y=2} K_{1,y}^\dagger K_{1,y} = \mathds{I}$. Observe that this HQMM has only 1 Kraus operator per observable, which means it models the state evolution as unitary. \\ 

Our results in Table \ref{prob-clock-table} demonstrate that a probability clock generates data that is hard for HMMs to model and that our iterative algorithm yields a simple HQMM that matches the predictive power of the original model.

\begin{table}[h]
\caption{ Performance of various HQMMs and HMMs learned from data generated by the probability clock model. HQMM parameters are given as $(n,s,w)$ and HMM parameters are given as $(n,s)$, where $n$ is the number of hidden states, $s$ is the number of observables, and $w$ is the number of Kraus operators per observable. (T) indicates the true model, (L) indicates learned models. P is the number of parameters. Both the mean and STD of the DA are indicated for training and test data.}\label{prob-clock-table}
  \centering
  \begin{tabular}{llll}
    \hline
    {\bf Model} & {\bf P} &
     {\bf Train DA} & {\bf Test DA} \\
      \hline
    $2,2,1-$HQMM (T) & 8 & $0.1642$ ($0.0089$) & $0.1632$ ($0.0111$)  \\
    $2,2,1-$HQMM (L) & 8 & $0.1640$ ($0.0088$) & {${0.1631}$ ($0.0111$)} \\
    \hline
    $2,2-$HMM (L) & 8 & $0.0851$ ($0.0074$) & $0.0833$ ($0.0131)$ \\
    $4,2-$HMM (L) & 24 & $0.1459$ ($0.0068$) & $0.1446$ ($0.0100$) \\
    $8,2-$HMM (L) & 80 & $0.1639$ ($0.0087$) & $0.1630$ ($0.0108$) \\
    \hline
  \end{tabular}
\end{table}

\subsection{\cite{monras2010hidden} 2-state HQMM}

\cite{monras2010hidden} present a 4-state, 4-output HMM with a loose lower bound requirement of 3 classical latent states that can be modeled by the following 2-state, 4-output HQMM:
\begin{alignat}{3}
\hat{K}_{1,1} &= \begin{pmatrix} \frac1{\sqrt{2}} & 0 \\ 0 & 0\end{pmatrix} \hspace{1cm}
&&\hat{K}_{1,2} = \begin{pmatrix} 0 & 0 \\ 0 & \frac1{\sqrt{2}}\end{pmatrix} \\
\hat{K}_{1,3} &= \begin{pmatrix} \frac1{2\sqrt{2}} & \frac1{2\sqrt{2}} \\ \frac1{2\sqrt{2}} & \frac1{2\sqrt{2}}\end{pmatrix} \hspace{1cm}
&&\hat{K}_{1,4} = \begin{pmatrix} \frac1{2\sqrt{2}} & -\frac1{2\sqrt{2}} \\ -\frac1{2\sqrt{2}} & \frac1{2\sqrt{2}}\end{pmatrix}
\end{alignat}

This model also treats state evolution as unitary since there is only 1 Kraus operator per observable. We generate data using this model, and our results in Table \ref{monras-table} show that our algorithm is capable of learning an HQMM that can match the DA of the original model, while the HMM needs more states to match the DA.

\begin{table}[h]
\caption{Performance of various HQMMs and HMMs on data generated by the \cite{monras2010hidden} model. HQMM parameters are given as $(n,s,w)$ and HMM parameters are given as $(n,s)$, where $n$ is the number of hidden states, $s$ is the number of observables, and $w$ is the number of Kraus operators per observable}\label{monras-table}
  \centering
  \begin{tabular}{llll}
    \toprule
     {\bf Model} & {\bf P} & 
      {\bf Train DA} & {\bf Test DA} \\
      \midrule
    $2,4,1-$HQMM (T) & 16 & $0.2505$ ($0.0037$) & $0.2516$ ($0.0063$)  \\
    $2,4,1-$HQMM (L) & 16 & $0.2501$ ($0.0085$) & $0.2512$ ($0.0064$) \\
    $2,4,2-$HQMM (L) & 32 &$0.2499$ ($0.0035$) & $0.2508$ ($0.0060$) \\
    \midrule
    $2,4-$HMM (L) & 12 & $0.0960$ ($0.0085$) & $0.0963$ ($0.0064)$ \\
    $3,4-$HMM (L) & 21 & $0.1387$ ($0.0067$) & $0.1416$ ($0.0070$) \\
    $4,4-$HMM (L) & 32 & $0.2504$ ($0.0037$) & $0.2515$ ($0.0062$) \\
    \bottomrule
  \end{tabular}
\end{table}

\subsection{A Fully Quantum HQMM}

In the previous two experiments, the HQMMs we used to generate data were also valid NOOMs since they used only real-valued entries. 
Here, we present the results of our algorithm on a fully quantum HQMM. Since we use complex-valued entries, there is no known way of writing our model as an equivalent-sized HMM, NOOM, or OOM.\\ 

We motivate this model with a physical system. Consider electron spin: quantized angular momentum that can either be `up' or `down' along whichever spatial axis the measurement is made, but not in between. There is no well-defined 3D vector describing electron spin along the 3 spatial dimensions, only `up' or `down' along a chosen axis of measurement (i.e., measurement basis). This is unlike classical angular momentum which can be represented by a vector with well-defined components in three spatial dimensions. Picking an arbitrary direction as the $z$-axis, we can write the electron's spin state in the $\{ +{\bf z}, - {\bf z}\}$ basis so that $\begin{bmatrix} 1 & 0 \end{bmatrix}^T$ is $|+{\bf z}\rangle$ and $\begin{bmatrix} 0 & 1 \end{bmatrix}^T$ is $|-{\bf z}\rangle$. But electron spin constitutes a two-state quantum system, so it can be in superpositions of the orthogonal `up' and `down' quantum states, which can be parameterized with $(\theta,\phi)$ and written as $|\psi\rangle = \cos\left(\frac{\theta}2\right)|+{\bf z}\rangle + e^{i\phi}\sin\left(\frac{\theta}2\right)|-{\bf z}\rangle$, where $0 \leq \theta \leq \pi$ and $0 \leq \phi \leq 2\pi$. The Bloch sphere (sphere with radius 1) is a useful tool to visualize qubits since it can map any two-state system to a point on the surface of the sphere using $(\theta, \phi)$ as polar and azimuthal angles. We could also have chosen $\{ +{\bf x}, - {\bf x}\}$ or $\{ +{\bf y}, - {\bf y}\}$, which can be written in our original basis:

\begin{alignat}{3}
|+{\bf x}\rangle &= \frac1{\sqrt{2}}|+{\bf z}\rangle + \frac1{\sqrt{2}}|-{\bf z}\rangle \hspace{1cm}&&\left(\theta=\frac{\pi}2, \phi = 0\right)\\
|-{\bf x}\rangle &= \frac1{\sqrt{2}}|+{\bf z}\rangle - \frac1{\sqrt{2}}|-{\bf z}\rangle \hspace{1cm}&&\left(\theta=\frac{\pi}2, \phi = \pi\right)\\
|+{\bf y}\rangle &= \frac1{\sqrt{2}}|+{\bf z}\rangle + \frac{i}{\sqrt{2}}|-{\bf z}\rangle \hspace{1cm}&&\left(\theta=\frac{\pi}2, \phi = \frac{\pi}2\right)\\
|-{\bf y}\rangle &= \frac1{\sqrt{2}}|+{\bf z}\rangle - \frac{i}{\sqrt{2}}|-{\bf z}\rangle \hspace{1cm}&&\left(\theta=\frac{\pi}2, \phi = \frac{3\pi}2\right)
\end{alignat}

Now consider the following process, inspired by the Stern-Gerlach experiment (\cite{gerlach1922experimentelle}) from quantum mechanics. 
We begin with an electron whose spin we represent in the $\{+ {\bf z}, - {\bf z}\}$ basis. At each time step, we pick one of the $x$, $y$, or $z$ directions uniformly and at random, and apply an inhomogeneous magnetic field along that axis. This is an act of measurement that collapses the electron spin to either `up' or `down' along that axis, which will deflect the electron in that direction. Let us use the following encoding scheme for the results of the measurement: $1$: $+ {\bf z}$, $2$: $- {\bf z}$, $3$: $+ {\bf x}$, $4$: $- {\bf x}$, $5$: $+ {\bf y}$, $6$: $- {\bf y}$. Consequently, at each time step, the observation tells us which axis we measured along, and whether the spin of the particle is now `up' or `down' along that axis. As an example, if we prepare an electron spin `up' along the $z$-axis, and observe the following sequence: $1,3,2,6$, it means that we applied the inhomogeneous magnetic field in the $z$-direction, then $x$-direction, then $z$-direction, and finally the $y$-direction, causing the electron spin state to evolve as $+{\bf z}, +{\bf x}, -{\bf z}, -{\bf y}$. \\

Note that transitions $1 \leftrightarrow 2$, $3 \leftrightarrow 4$, and $5 \leftrightarrow 6$ are not allowed, since there are no spin-flip operations in our process. Admittedly, this is a slightly contrived example, since normally we think of a hidden state that evolves according to some rules, producing noisy observation. Here, we select the observation (down to the pair, $(1,2)$, $(3,4)$, $(5,6)$) that we wish to observe, and that tells us how the `hidden state' evolves as described by a chosen basis. \\

This model is related to the 2-state HQMM requiring 3 classical states described in \cite{monras2010hidden}. It is still a 2-state system, but we add two new Kraus operators with complex entries and renormalize:

\begin{alignat}{3}
\hat{K}_{1,1} &= \begin{pmatrix} \frac1{\sqrt{3}} & 0 \\ 0 & 0\end{pmatrix} \hspace{1cm}
&&\hat{K}_{1,2} = \begin{pmatrix} 0 & 0 \\ 0 & \frac1{\sqrt{3}}\end{pmatrix} \\
\hat{K}_{1,3} &= \begin{pmatrix} \frac1{2\sqrt{3}} & \frac1{2\sqrt{3}} \\ \frac1{2\sqrt{3}} & \frac1{2\sqrt{3}}\end{pmatrix} \hspace{1cm}
&&\hat{K}_{1,4} = \begin{pmatrix} \frac1{2\sqrt{3}} & -\frac1{2\sqrt{3}} \\ -\frac1{2\sqrt{3}} & \frac1{2\sqrt{3}}\end{pmatrix}\\
\hat{K}_{1,5} &= \begin{pmatrix} \frac1{2\sqrt{3}} & -\frac{i}{2\sqrt{3}} \\ \frac{i}{2\sqrt{3}} & \frac1{2\sqrt{3}}\end{pmatrix} \hspace{1cm}
&&\hat{K}_{1,6} = \begin{pmatrix} \frac1{2\sqrt{3}} & \frac{i}{2\sqrt{3}} \\ -\frac{i}{2\sqrt{3}} & \frac1{2\sqrt{3}}\end{pmatrix}
\end{alignat}

Physically, Kraus operators $\hat{K}_{1,1}$ and $\hat{K}_{1,2}$ keep the spin along the $z$-axis, Kraus operators $\hat{K}_{1,3}$ and $\hat{K}_{1,4}$ rotate the spin to lie along the $x$-axis, while Kraus operators $\hat{K}_{1,5}$ and $\hat{K}_{1,6}$ rotate the spin to lie along the $y$-axis. Following the approach of \cite {monras2010hidden}, we write down an equivalent 6-state HMM, and compute the rank of a Hankel matrix with the statistics of this process, yielding a requirement of 4 classical states as a weak lower bound. \\

We present the results of our learning algorithm applied to data generated by this model in Table \ref{full-quant-table}. We find that our algorithm can learn a 2-state HQMM (same size as the model that generated the data) with predictive power matched only by a 6-state HMM.

\begin{table}[h]
\caption{Performance of various HQMMs and HMMs on the fully quantum HQMM. HQMM parameters are given as $(n,s,w)$ and HMM parameters are given as $(n,s)$, where $n$ is the number of hidden states, $s$ is the number of observables, and $w$ is the number of Kraus operators per observable}\label{full-quant-table} 
  \centering
  \begin{tabular}{llll}
    \hline
    {\bf Model} & {\bf P} &
     {\bf Train DA} & {\bf Test DA} \\
      \hline
    $2,6,1-$HMM (T) & 24 & $0.1303$ ($0.0042$) & $0.1303$ ($0.0047$)  \\
    $2,6,1-$HQMM (L) & 24 & $0.1303$ ($0.0042$) & $0.1301$ ($0.0047$) \\
    \hline
    $2,6-$HMM (L) & 16 & $0.0327$ ($0.0038$) & $0.0328$ ($0.0033$)  \\
    $3,6-$HMM (L) & 27 & $0.0522$ ($0.0043$) & $0.0530$ ($0.0040$) \\
    $4,6-$HMM (L) & 40  & $0.0812$ ($0.0042$) & $0.0822$ ($0.0045$) \\
    $5,6-$HMM (L) & 55  & $0.0967$ ($0.0042$) & $0.0967$ ($0.0045$) \\
    $6,6-$HMM (L) & 72 & $0.1305$ ($0.0042$) & $0.1301$ ($0.0049$)  \\
    \hline
  \end{tabular}
\end{table}


\subsection{Synthetic Data from a hand-written HMM}

We have shown that we can generate data using HQMMs that classical HMMs with the same number of hidden states struggle to model. In this section, we explore how well HQMMs can model data generated by a classical HMM. In general, randomly generated HMMs generate data that is hard to predict (i.e., DA closer to 0), so we hand-author an arbitrary, well-behaved HMM with full-rank transition matrix ${\bf A}$ and full-rank emission matrix ${\bf C}$ to compare HQMM learning with EM for HMMs:

\begin{equation}
{\bf A } = \begin{bmatrix} 0.8 & 0.01 & 0 & 0.1 & 0.3 & 0\\ 0.02 & 0.02 & 0.1 & 0.15 & 0.05 & 0\\ 0.08 & 0.03 & 0.1 & 0.4 & 0.05 & 0.5\\ 0.05 & 0.04 & 0.5 & 0.35 & 0 & 0.5 \\ 0.03 & 0.5 & 0.03 & 0 & 0.6 & 0 \\ 0.02 & 0.4 & 0.27 & 0 & 0 & 0\end{bmatrix}, \,\,\,\,\,\,\,\,\,
{\bf C } = \begin{bmatrix} 0.2 & 0 & 0.05 & 0.95 & 0.01 & 0.05 \\ 0.7 & 0.1 & 0.05 & 0.01 & 0.05 & 0.05 \\ 0.05 & 0.8 & 0.1 & 0.02 & 0.05 & 0.04 \\ 0.04 & 0.04 & 0.02 & 0 & 0.84 & 0.11 \\ 0.01 & 0.03 & 0.7 & 0.01 & 0.02 & 0.2 \\ 0 & 0.03 & 0.08 & 0.01 & 0.03 & 0.55 \end{bmatrix}
\end{equation}

Our results are presented in Table \ref{synth-table}. We find that small HQMMs outperform HMMs with the same number of hidden states, although the parameter count ends up being larger. However, as model size increases, training becomes quite slow, and our HQMMs are over-parameterized, becoming prone to local optima, and EM for HMMs may work better in practice on HMM-generated data. Interestingly, even though our scheme in Section \ref{qcirc-section} requires $w=n$ to simulate HMMs with HQMMs, empirically, we find that we are able to learn reasonable models with $w < n$.
 
\begin{table}[h]
\caption{Performance of various HQMMs and HMMs on synthetic data generated by an HMM. HQMM parameters are given as $(n,s,w)$ and HMM parameters are given as $(n,s)$, where $n$ is the number of hidden states, $s$ is the number of observables, and $w$ is the number of Kraus operators per observable}\label{synth-table}
  \centering
  \begin{tabular}{llll}
    \toprule
   {\bf Model} & {\bf P} & 
      {\bf Train DA} & {\bf Test DA} \\
      \midrule
    $6,6-$HMM (T) & 72 & $0.1838$ ($0.0095$) & $0.1903$ ($0.0071$)  \\
    $2,6,1-$HQMM (L) & 24 & $0.1597$ ($0.0088$) & $0.1659$ ($0.0073$) \\
    $3,6,1-$HQMM (L) & 54 & $0.1655$ ($0.0101$) & $0.1715$ ($0.0085$)\\
    $4,6,1-$HQMM (L) & 96 & $0.1732$ ($0.0103$) & $0.1772$ ($0.0103$)  \\
    $5,6,1-$HQMM (L) & 150 & $0.1680$ ($0.0093$) & $0.1706$ ($0.0084$)  \\
    $5,6,2-$HQMM (L) & 300 & $0.1817$ ($0.0096$) & $0.1863$  ($0.0069$) \\
    $5,6,3-$HQMM (L) & 450 & $0.1817$ ($0.0093$) & $0.1866$ ($0.0064$) \\
    $5,6,5-$HQMM (L) & 750 & $0.1821$ ($0.0095$) & $0.1877$ ($0.0060$)  \\
    $6,6,1-$HQMM (L) & 216 & $0.1713$ ($0.0113$) & $0.1708$ ($0.0079$) \\
    $6,6,2-$HQMM (L) & 432 & $0.1817$ ($0.0096$) & $0.1870$ ($0.0070$)\\
    \midrule
    $2,6-$HMM (L) & 16 & $0.1282$ ($0.0074$) & $0.1314$ ($0.0062$)  \\
    $3,6-$HMM (L) & 27 & $0.1555$ ($0.0097$) & $0.1625$ ($0.0073$) \\
    $4,6-$HMM (L) & 40  & $0.1667$ ($0.0099$) & $0.1732$ ($0.0068$) \\
    $5,6-$HMM (L) & 55  & $0.1751$ ($0.0097$ & $0.1816$ ($0.0070$) \\
    $6,6-$HMM (L) & 72 & $0.1841$ ($0.0095$) & $0.1901$ ($0.0070$)  \\
    \bottomrule
  \end{tabular}
\end{table}








\section{Conclusion}

We formulated and parameterized hidden quantum Markov models by first finding quantum circuits to implement HMMs, reducing them to their Kraus operator representation, and then relaxing some constraints. We showed how quantum analogues of classical conditioning and marginalization can be implemented, and indeed these methods are general enough to allow us to construct quantum versions of any probabilistic graphical model. We also proposed an iterative maximum-likelihood algorithm to learn the Kraus operators for HQMMs. We demonstrated that our algorithm could successfully learn HQMMs that were shown to (theoretically) better model certain sequences in the literature. While our HQMMs cannot model data any better than a sufficiently large HMM, we find that HQMMs can often better model the same data with fewer hidden states. 
%
Future work could look at optimizing our algorithm to scale on larger datasets, and at the performance of HQMMs in areas like natural language processing or finance, or quantum versions of existing graphical models. 
We speculate that quantum models could lead to improvements in these areas where `quantum' effects may be able to better simulate the dynamic processes.

\subsubsection*{Acknowledgements}
We would like to thank Theresa W. Lynn at Harvey Mudd College for her inputs and feedback on this work.

\small
\bibliographystyle{plainnat}
\bibliography{bibliography}

\begin{thebibliography}{12}
\providecommand{\natexlab}[1]{#1}
\providecommand{\url}[1]{\texttt{#1}}
\expandafter\ifx\csname urlstyle\endcsname\relax
  \providecommand{\doi}[1]{doi: #1}\else
  \providecommand{\doi}{doi: \begingroup \urlstyle{rm}\Url}\fi

\bibitem[Biamonte et~al.(2016)Biamonte, Wittek, Pancotti, Rebentrost, Wiebe,
  and Lloyd]{biamonte2016quantum}
Jacob Biamonte, Peter Wittek, Nicola Pancotti, Patrick Rebentrost, Nathan
  Wiebe, and Seth Lloyd.
\newblock Quantum machine learning.
\newblock \emph{arXiv preprint arXiv:1611.09347}, 2016.

\bibitem[Clark et~al.(2015)Clark, Huang, Barlow, and Beige]{clark2015hidden}
Lewis~A Clark, Wei Huang, Thomas~M Barlow, and Almut Beige.
\newblock Hidden quantum markov models and open quantum systems with
  instantaneous feedback.
\newblock In \emph{ISCS 2014: Interdisciplinary Symposium on Complex Systems},
  pages 143--151. Springer, 2015.

\bibitem[Gerlach and Stern(1922)]{gerlach1922experimentelle}
Walther Gerlach and Otto Stern.
\newblock Der experimentelle nachweis der richtungsquantelung im magnetfeld.
\newblock \emph{Zeitschrift f{\"u}r Physik}, 9\penalty0 (1):\penalty0 349--352,
  1922.

\bibitem[Jaeger(2000)]{jaeger2000observable}
Herbert Jaeger.
\newblock Observable operator models for discrete stochastic time series.
\newblock \emph{Neural Computation}, 12\penalty0 (6):\penalty0 1371--1398,
  2000.

\bibitem[Johnston(2016)]{qetlab}
Nathaniel Johnston.
\newblock {QETLAB}: A {MATLAB} toolbox for quantum entanglement, version 0.9.
\newblock \url{http://qetlab.com}, January 2016.

\bibitem[M.~Zhao(2007)]{noomreport}
H.~Jaeger M.~Zhao.
\newblock Norm observable operator models.
\newblock Technical report, Jacobs University, 2007.

\bibitem[Monras et~al.(2010)Monras, Beige, and Wiesner]{monras2010hidden}
Alex Monras, Almut Beige, and Karoline Wiesner.
\newblock Hidden quantum markov models and non-adaptive read-out of many-body
  states.
\newblock \emph{arXiv preprint arXiv:1002.2337}, 2010.

\bibitem[Rabiner(1989)]{rabiner1989tutorial}
Lawrence~R Rabiner.
\newblock A tutorial on hidden markov models and selected applications in
  speech recognition.
\newblock \emph{Proceedings of the IEEE}, 77\penalty0 (2):\penalty0 257--286,
  1989.

\bibitem[Schuld et~al.(2015{\natexlab{a}})Schuld, Sinayskiy, and
  Petruccione]{schuld2015introduction}
Maria Schuld, Ilya Sinayskiy, and Francesco Petruccione.
\newblock An introduction to quantum machine learning.
\newblock \emph{Contemporary Physics}, 56\penalty0 (2):\penalty0 172--185,
  2015{\natexlab{a}}.

\bibitem[Schuld et~al.(2015{\natexlab{b}})Schuld, Sinayskiy, and
  Petruccione]{schuld2015simulating}
Maria Schuld, Ilya Sinayskiy, and Francesco Petruccione.
\newblock Simulating a perceptron on a quantum computer.
\newblock \emph{Physics Letters A}, 379\penalty0 (7):\penalty0 660--663,
  2015{\natexlab{b}}.

\bibitem[Wiebe et~al.(2016)Wiebe, Kapoor, and Svore]{NIPS2016_6401}
Nathan Wiebe, Ashish Kapoor, and Krysta Svore.
\newblock Quantum perceptron models.
\newblock In D.~D. Lee, M.~Sugiyama, U.~V. Luxburg, I.~Guyon, and R.~Garnett,
  editors, \emph{Advances in Neural Information Processing Systems 29}, pages
  3999--4007. Curran Associates, Inc., 2016.
\newblock URL
  \url{http://papers.nips.cc/paper/6401-quantum-perceptron-models.pdf}.

\bibitem[Zhao and Jaeger(2010)]{zhao2010norm}
Ming-Jie Zhao and Herbert Jaeger.
\newblock Norm-observable operator models.
\newblock \emph{Neural computation}, 22\penalty0 (7):\penalty0 1927--1959,
  2010.

\end{thebibliography}

\setcounter{section}{1}
\renewcommand\thesection{\Alph{section}}
\renewcommand\thesection{\Alph{subsection}}
\clearpage
\section*{APPENDIX}

\subsection{Tensor Product and Partial Trace as Matrix Operations}
Here we go into more depth on how we construct matrices $W$, $V_y$ and $V_w$ to perform the tensor product and partial trace operations for use in our Algorithm \ref{alg:hqmm}.

\subsubsection{Tensor Product}
We construct a matrix $W$ that performs tensor product with an $s \times s$ density matrix $\hat{\rho}_B$ with all zeros, except $\hat{\rho}_{1,1} = 1$, i.e., $\hat{\rho}_B = \begin{pmatrix} 1 & 0 & \ldots & 0 \\ 0 & 0 & \ldots & 0 \\\vdots & \vdots & \ddots & 0 \\ 0 & 0 & 0 & 0\end{pmatrix}_{s \times s}$. \\

Observe that for an $n \times n$ density matrix $\hat{\rho}_A$, we the tensor product yields an $ns \times ns$ matrix $\hat{\rho}_{AB} = \hat{\rho}_A \otimes \hat{\rho}_B$. Thus, our matrix $W$ will be an $ns \times n$ matrix, such that $\hat{\rho}_A \otimes \hat{\rho}_B = W\hat{\rho}_AW^\dagger$. \\

To construct $W$, take $n$ of $s \times n$ matrices of zeros, for the $i$th among those $n$ matrices, place `$1$' on the first row and $i$th column. Then stack all of those matrices vertically to obtain the $ns \times n$ matrix $W$.\\

\underline{Example}
If we have a $3\times 3$ density matrix we wished to tensor with a $4 \times 4$ density matrix, we construct $W$ such that:
\begin{equation}
W = \begin{bmatrix} 1 & 0 & 0 \\ 0 &  0 & 0 \\ 0 &  0 & 0 \\ 0 &  0 & 0 \\ 0 & 1 & 0 \\ 0 & 0 & 0 \\ 0 & 0 & 0 \\ 0 & 0 & 0 \\ 0 & 0 & 1 \\ 0 & 0 & 0 \\ 0 & 0 & 0 \\ 0 & 0 & 0 \end{bmatrix}_{12 \times 3}
\end{equation}
Then, we find that:
\begin{equation}
\hat{\rho}_A \otimes \begin{bmatrix} 1 & 0 & 0 & 0 \\ 0 & 0 & 0 & 0 \\ 0 & 0 & 0 & 0 \\ 0 & 0 & 0 & 0\end{bmatrix} = W\hat{\rho}_AW^\dagger
\end{equation}

\subsubsection{Partial Trace}
The partial trace cannot ordinarily be implemented with a single matrix operation. However, if a projection operator has just been applied, this operation becomes trivial and easy to perform with a matrix multiplication, i.e., $\text{tr}_B \left(\hat{P}_y \hat{\rho}_{AB} \hat{P}_y^\dagger \right) = V_y\hat{P}_y \hat{\rho}_{AB}\hat{P}_y^\dagger V_y^\dagger$. On the other hand, if we wish to take the partial trace without applying a projection operator, i.e., without a measurement of one of the two subsystems, we must take a sum over these matrices like so: $\text{tr}_A \left(\hat{\rho}_{AB} \right) = \sum_w V_w \hat{\rho}_{AB} V_w^\dagger$. The subscript of `$\text{tr}$' tells us which particle we are tracing over. \\

\underline{Partial Trace after Projection}
Here, we will assume that a projection operator $\hat{P}_y$ corresponding to an observation on the second particle in the same basis was applied on the joint state of a system prior to the partial trace. If this is not the case, we simply construct all matrices $V_y$ for each observation and take a sum as previously described.\\

The construction of this matrix $V_y$ is straightforward. We take $s$ of $n \times s$ matrices of zeros, and for the $i$th of these $s$ matrices, place `$1$' on the $y$th column and $i$th row. Then, concatenate these matrices horizontally to obtain $V_y$. \\

\underline{Example}
If we have a $12\times 12$ density matrix describing the joint state of a $3$-state particle and $4$-state particle, we can construct $V_2$ to trace over the second particle after applying a projection operator $\hat{P}_2$ to be:
\setcounter{MaxMatrixCols}{20}
\begin{equation}
V_2 = \begin{bmatrix} 0 & 1 & 0 & 0 & 0 & 0 & 0 & 0 & 0 & 0 & 0 & 0 \\ 0 & 0 & 0 & 0 & 0 & 1 & 0 & 0 & 0 & 0 & 0 & 0  \\ 0 & 0 & 0 & 0 & 0 & 0 & 0 & 0 & 0 & 1 & 0 & 0 \end{bmatrix}_{3 \times 12}
\end{equation}

Then, we find that if we have applied a projection operator:
\begin{equation} 
\text{tr}_B\left( \hat{P}_2 \hat{\rho}_{AB} \hat{P}_2^\dagger \right) = V_2\hat{P}_2 \hat{\rho}_{AB} \hat{P}_2^\dagger V_2^\dagger
\end{equation}

\underline{Partial Trace without Projection}
Here, we assume that no measurement/projection has been made, since this is how we use it in the algorithm. If this is not the case and there a projection operator was applied, forgo the sum and simply apply the $V_w$ corresponding to the measurement. \\

To perform partial trace where there has been no observation, we must construct a set of matrices $V_w$, which we apply and then sum over. The construction of each matrix $V_w$ is as follows. We take $s$ of $s \times n$ matrices of zeros, except the $w$th out these $s$ matrices which is an identity matrix. Then concatenate these matrices horizontally to obtain $V_w$. \\

\underline{Example}
If we have a $12\times 12$ density matrix describing the joint state of a $3$-state particle and $4$-state particle, we can construct $V_w$ to trace over the first particle as:
\setcounter{MaxMatrixCols}{20}
\begin{equation}
\begin{split}
V_1 &= \begin{bmatrix} 1 & 0 & 0 & 0 & 0 & 0 & 0 & 0 & 0 & 0 & 0 & 0 \\ 0 & 1 & 0 & 0 & 0 & 0 & 0 & 0 & 0 & 0 & 0 & 0  \\ 0 & 0 & 1 & 0 & 0 & 0 & 0 & 0 & 0 & 0 & 0 & 0 \\ 0 & 0 & 0 & 1 & 0 & 0 & 0 & 0 & 0 & 0 & 0 & 0 \end{bmatrix}_{4 \times 12} \\
V_2 &= \begin{bmatrix} 0 & 0 & 0 & 0 & 1 & 0 & 0 & 0 & 0 & 0 & 0 & 0 \\ 0 & 0 & 0 & 0 & 0 & 1 & 0 & 0 & 0 & 0 & 0 & 0  \\ 0 & 0 & 0 & 0 & 0 & 0 & 1 & 0 & 0 & 0 & 0 & 0 \\ 0 & 0 & 0 & 0 & 0 & 0 & 0 & 1 & 0 & 0 & 0 & 0 \end{bmatrix}_{4 \times 12} \\
V_3 &= \begin{bmatrix} 0 & 0 & 0 & 0 & 0 & 0 & 0 & 0 & 1 & 0 & 0 & 0 \\ 0 & 0 & 0 & 0 & 0 & 0 & 0 & 0 & 0 & 1 & 0 & 0  \\ 0 & 0 & 0 & 0 & 0 & 0 & 0 & 0 & 0 & 0 & 1 & 0 \\ 0 & 0 & 0 & 0 & 0 & 0 & 0 & 0 & 0 & 0 & 0 & 1 \end{bmatrix}_{4 \times 12}
\end{split}
\end{equation}

Then, we find that:
\begin{equation} 
\hat{\rho}_B = \text{tr}_A\left( \hat{\rho}_{AB} \right) = \sum_{w=1}^3 V_w \hat{\rho}_{AB} V_w^\dagger
\end{equation}

\subsection{Factorizing Unitary Matrices into {\bf H} Matrices} The proof of this theorem is a generalization of the proof found in \cite{noomreport}.

\begin{lemma} \emph{For any vector $\vec{x} \in \mathds{C}^n$ where $n \geq 2$, there exists a matrix ${\bf A}$ that is a product of $H$ matrices, such that ${\bf A}\vec{x} = \|\vec{x}\|\vec{e}_1$ where $\vec{e}_1 = \begin{bmatrix} 1 & 0 & \ldots & 0\end{bmatrix}_{1\times n}^T$ (unit vector in $\mathds{R}^n$).}\end{lemma}

\begin{proof} Consider an arbitrary vector $\vec{x} \in \mathds{C}^n$, written as $\vec{x} = \begin{bmatrix} x_1 & x_2 & \ldots & x_n\end{bmatrix}_{1\times n}^T$. Let us define $y_2 = \sqrt{\|x_1\|^2 + \|x_2\|^2}$ and parameterize the entries $x_1$ and $x_2$ in $\vec{x}$ with $\alpha_2$ and $\beta_2$ so as to write:
\begin{equation}
\begin{split}
x_1 = y_2e^{i\beta_2}\cos(\alpha_2)\\
x_2 = y_2e^{i\beta_2}\sin(\alpha_2)
\end{split}
\end{equation}
Now consider the action of ${\bf H}_1(1,2,\alpha_2,-2\beta_2,0,0)$ on $\vec{x}$:
\begin{equation}
{\bf H}_1\vec{x} =\begin{bmatrix}  e^{-i\beta_2}\cos\alpha_2 & e^{-i\beta_2}\sin\alpha_2 & 0 & \cdots & 0  \\ -e^{-i\beta_2}\sin\alpha_2 & e^{-i\beta_2}\cos\alpha_2  & \vdots & \vdots  &  \vdots  \\ 0 & 0 & 1 & \cdots & 0  \\ \vdots &  & \vdots & \ddots  &  \vdots \\ 0 & 0 & 0 & \cdots & 1  \end{bmatrix}\begin{bmatrix} y_2e^{i\beta_2}\cos(\alpha_2) \\ y_2e^{i\beta_2}\sin(\alpha_2) \\ x_3 \\ \vdots\\ x_n\end{bmatrix} = \begin{bmatrix} y_2 \\ 0 \\ x_3 \\ \vdots\\ x_n\end{bmatrix}
\end{equation}

Next, we can define $y_3 = \sqrt{\|y_2\|^2 + \|x_3\|^2}$ and parameterize $y_2$ and $x_3$ using $\alpha_3$ and $\beta_3$, just like we previously. We can then apply ${\bf H}_2(1,3,\alpha_3,-2\beta_3,0,0)$, and we find that:
\begin{equation}
{\bf H_2}{\bf H_1}\vec{x} = \begin{bmatrix} y_3 \\ 0 \\ 0 \\ x_4 \\ \vdots\\ x_n\end{bmatrix}
\end{equation}
Following this pattern, we can construct a sequence of ${\bf H}$ matrices such that ${\bf H}_{n-1}\ldots{\bf H}_2{\bf H}_1\vec{x} = \begin{bmatrix} y_n & 0 & 0 & \ldots & 0\end{bmatrix}^T$. Observe that $y_n = \sqrt{\|y_{n-1}\|^2 + \|x_n\|^2} = \sqrt{\|y_{n-2}\|^2 + \|x_{n-1}\|^2 +  \|x_n\|^2} = \sqrt{\|x_1\|^2 + \ldots \|x_n\|^2} = \|\vec{x}\|$. Thus, with ${\bf A} = {\bf H}_{n-1}\ldots{\bf H}_2{\bf H}_1$, we have shown that there exists a matrix ${\bf A}$ that is a product of $H$ matrices, such that ${\bf A}\vec{x} = \|\vec{x}\|\vec{e}_1$.\end{proof}

\begin{lemma} \emph{Any 2x2 unitary matrix ${\bf A}$ can be written as ${\bf H}(1,2,\theta,\phi,\psi,\delta)$.}\end{lemma}

\begin{proof} A generalized 2x2 unitary matrix is written as:
\begin{equation}
\begin{bmatrix} e^{\nicefrac{i\phi}2}e^{i\psi}\cos\theta &  e^{\nicefrac{i\phi}2}e^{i\delta}\sin\theta \\ -e^{\nicefrac{i\phi}2}e^{-i\delta}\sin\theta & e^{\nicefrac{i\phi}2}e^{-i\psi}\cos\theta \end{bmatrix}
\end{equation}
which is exactly ${\bf H}(1,2,\theta,\phi,\psi,\delta)$.\end{proof}

\begin{theorem} \emph{A matrix $\hat{\bf U}$ is unitary if and only if it can be written as a product of ${\bf H}(i,j,\theta,\phi,\psi,\delta)$ matrices with the following form,  where $i,j$ denote the rows and columns with special entries:}
\begin{align}
&{\bf H}(i,j,\theta,\phi,\psi,\delta) = \begin{bmatrix} 1 & \cdots & 0 & \cdots & 0 & \cdots & 0 \\ \vdots & \ddots & \vdots &  &  \vdots & & \vdots \\ 0 & \cdots & e^{\nicefrac{i\phi}2}e^{i\psi}\cos\theta & \cdots & e^{\nicefrac{i\phi}2}e^{i\delta}\sin\theta & \cdots & 0 \\ \vdots &  & \vdots & \ddots  &  \vdots & & \vdots \\ 0 & \cdots & -e^{\nicefrac{i\phi}2}e^{-i\delta}\sin\theta & \cdots & e^{\nicefrac{i\phi}2}e^{-i\psi}\cos\theta & \cdots & 0 \\ \vdots & & \vdots &  &  \vdots & \ddots & \vdots \\ 0 & \cdots & 0 & \cdots & 0 & \cdots & 1\end{bmatrix}
\end{align}
\end{theorem}

\begin{proof} We will prove both the forward and reverse directions:
\begin{enumerate}
\item \emph{A matrix $\hat{\bf U}$ is unitary if it can be written as a product of ${\bf H}$ matrices.}

Observe that matrix ${\bf H}$ is unitary, since ${\bf H}^\dagger {\bf H} = \mathds{I}$. A product of unitary matrices is itself unitary, hence a matrix $\hat{\bf U}$ that is a product of these ${\bf H}$ matrices is unitary.

\item \emph{If a matrix $\hat{\bf U}$ is unitary, it can be written as a product of ${\bf H}$ matrices.}

We will give a proof by induction. We want to show that any $n\times n$ $\hat{\bf U}$ unitary matrix can be written as product of ${\bf H}$ matrices.\\

{\bf Base Case} When $n=2$, i.e., for a $2 \times 2$ unitary matrix, we know that it can be written as product of ${\bf H}$ matrices from Lemma 2.\\

{\bf Inductive Hypothesis} Assume that the claim holds for $n=k$, i.e., any $k \times k$ unitary matrix can be written as a product of ${\bf H}$ matrices.\\

With $n=k+1$, consider an arbitrary $(k+1)\times (k+1)$ unitary matrix ${\hat {\mathbf{U}}} = \begin{bmatrix} \vec{u}_1 & \vec{u}_2 & \ldots & \vec{u}_{k+1} \end{bmatrix}$ where $\vec{u}_i$ is the $i$th column. Since $\hat{\mathbf{U}}$ is unitary, $\|\vec{u}_i\|= 1$ for $1\leq i \leq k+1$. Then, by Lemma 1, we have a matrix ${\bf A}$ that is a product of ${\bf H}$ matrices such that ${\bf A}\vec{u_1} = \|\vec{u}_1\|\vec{e}_1 = \vec{e}_1$.

Using this matrix, we find that $\hat{\bf U}' = {\bf A}\hat{\bf U} = \begin{bmatrix} 1 & \vec{C} \\ \vec{0} & {\bf V} \end{bmatrix}$ where $\vec{0}$ represents a $k \times 1$ column vector, $\vec{C}$ represents a $1 \times k$ row vector, and ${\bf V}$ represents  a $k\times k$ matrix. But $\hat{\bf U}'$ is unitary, so $\hat{\bf U}'(\hat{\bf U}')^\dagger = \mathds{I}$, which means ${\bf V}{\bf V}^\dagger = \mathds{I}_{k \times k}$ and $\vec{C} = \vec{0}$.\\

{\bf Inductive Step} From the inductive hypothesis, we know that ${\bf V}$ can be written as a product of ${\bf H}$ matrices, so let us write ${\bf V} = {\bf H}_k,\ldots,{\bf H}_1$. Next, we take each of these $k \times k$ ${\bf H}$ matrices and pad them to obtain $(k+1)\times (k+1)$ matrices ${\bf H}_i' = \begin{bmatrix} 1 & 0 \\ 0 & {\bf H}_i\end{bmatrix}$. Then, we see that ${\bf H}_k',\ldots,{\bf H}_1' = \begin{bmatrix} 1 & \vec{0} \\ \vec{0} & {\bf V} \end{bmatrix} = \hat{\bf U}'$.

Finally, we can write our arbitrary unitary matrix $\hat{\bf U} = {\bf A}^{-1}{\bf H}_k',\ldots,{\bf H}_1'$, which is indeed a product of ${\bf H}$ matrices. Hence, we have shown that any unitary matrix can be written as a product of ${\bf H}$ matrices.

\end{enumerate}
\end{proof}

\end{document}